\documentclass{article}
\usepackage[numbers]{natbib}
\usepackage[a4paper, total={6in, 8in}]{geometry}
\usepackage{hyperref}       
\usepackage{url}            
\usepackage{amsfonts}       
\usepackage{nicefrac}       
\usepackage{microtype}      

\usepackage{amsmath}        
\usepackage{amsthm}
\usepackage{enumitem} 
\usepackage{graphicx}

\newtheorem{theorem}{Theorem}[section]
\newtheorem{lemma}{Lemma}[section]

\def\P#1{\mathbb{P}\left({#1}\right)}
\def\Px#1{\mathbb{P}^\mathbf{x}\left({#1}\right)}
\def\Ex#1{\mathbb{E}^\mathbf{x}\left[{#1}\right]}
\def\gcd#1{g_{\text{cd}}({#1})}
\def\intX{\int_\mathcal{X}}

\begin{document}

\begin{center}
{\large\bf Convergence of Contrastive Divergence}\\
{\large\bf with Annealed Learning Rate in Exponential Family}\\
Bai Jiang, Tung-yu Wu, and Wing H. Wong\\
{\it Stanford University}
\end{center}

\centerline{\bf Abstract}
In our recent paper, we showed that in exponential family, contrastive divergence (CD) with fixed learning rate will give asymptotically consistent estimates \cite{wu2016convergence}. In this paper, we establish consistency and convergence rate of CD with annealed learning rate $\eta_t$. Specifically, suppose CD-$m$ generates the sequence of parameters $\{\theta_t\}_{t \ge 0}$ using an i.i.d. data sample $\mathbf{X}_1^n \sim p_{\theta^*}$ of size $n$, then $\delta_n(\mathbf{X}_1^n) = \limsup_{t \to \infty} \Vert \sum_{s=t_0}^t \eta_s \theta_s / \sum_{s=t_0}^t \eta_s - \theta^* \Vert$ converges in probability to 0 at a rate of $1/\sqrt[3]{n}$. The number ($m$) of MCMC transitions in CD only affects the coefficient factor of convergence rate. Our proof is not a simple extension of the one in \cite{wu2016convergence}. which depends critically on the fact that $\{\theta_t\}_{t \ge 0}$ is a homogeneous Markov chain conditional on the observed sample $\mathbf{X}_1^n$. Under annealed learning rate, the homogeneous Markov property is not available and we have to develop an alternative approach based on super-martingales. Experiment results of CD on a fully-visible $2\times 2$ Boltzmann Machine are provided to demonstrate our theoretical results.

\section{Introduction}
Consider a statistical model of the form
\[
p_\theta(x) = e^{-E(x;\theta) - \Lambda(\theta)},
\]
where $E(x;\theta)$ is the energy function and $\Lambda(\theta) = \log \intX e^{-E(y;\theta)}dy$ is the log-partition function. Given an i.i.d. sample $\mathbf{X}_1^n = \{X_1,\dots,X_n\}$ from $p_{\theta^*}$, we are interested in the estimation of $\theta^*$. It may be achieved by gradient ascent, i.e.,
\begin{equation} \label{alg: gradient ascent}
g(\theta) = \frac{1}{n}\sum_{i=1}^n \phi(X_i;\theta) - \nabla \Lambda(\theta),  \quad\theta_{t+1} = \theta_t + \eta_t g(\theta_t).
\end{equation}
where $g(\theta)$ denotes the gradient of the log-likelihood function, $\phi(x; \theta) = -\nabla_\theta E(x;\theta)$, and $\eta_t$ is the learning rate.

In many important models, $\nabla \Lambda(\theta) = \intX \phi(x; \theta) p_\theta(x)dx$ is not available in a close form. And we have to use Markov Chain Monte Carlo (MCMC) method such as Metropolis-Hasting algorithm and Gibbs sampling to approximate it. Unfortunately, it is computationally prohibitive to obtain accurate approximation in each step of the iteration (\ref{alg: gradient ascent}) by MCMC. To address this problem, \citet{hinton2002training} suggests running the Gibbs sampling or Metropolis-Hasting update in the MCMC for only a finite number ($m$) of transitions starting from every single datum $X_i$ for $i=1,\dots,n$,
\[
X_i \overset{k_\theta}{\longrightarrow} X_i^{(1)}  \overset{k_\theta}{\longrightarrow} X_i^{(2)} \overset{k_\theta}{\longrightarrow} \dots  \overset{k_\theta}{\longrightarrow} X_i^{(m)},
\]
and approximating $\nabla \Lambda(\theta)$ by $\frac{1}{n}\sum_{i=1}^n \phi(X_i^{(m)}; \theta)$. Hinton called this method Contrastive Divergence (CD) learning. Specifically, the CD gradient and update equation are given by
\begin{equation} \label{alg: CD gradient ascent}
\gcd{\theta} = \frac{1}{n}\sum_{i=1}^n \phi(X_i; \theta) - \frac{1}{n}\sum_{i=1}^n \phi(X_i^{(m)}; \theta), \quad \theta_{t+1} = \theta_t + \eta_t \gcd{\theta_t}. 
\end{equation}
In 2006 \citet{hinton2006fast} used CD to train Restricted Boltzmann Machines in deep belief networks. Since then CD has played an important role in the development of deep learning. It has also been successfully applied to other types of Markov Random Fields \cite{he2004multiscale, roth2005fields}.

Despite CD's empirical success, examples in \cite{mackay2001failures, teh2003energy, williams2002analysis} have shown that it does not always converge to the true parameter. \citet{yuille2005convergence} related CD to the stochastic approximation literature and derives conditions (\ref{eqn: Yuillie1}) and (\ref{eqn: Yuillie2}) which ensure convergence (Result 4 in \cite{yuille2005convergence})
\begin{align} 
& \frac{1}{n} \sum_{i=1}^n \intX \phi(y; \theta^*)k_{\theta^*}^m(X_i, y)dy = \nabla \Lambda(\theta^*) \label{eqn: Yuillie1}\\
& (\theta - \theta^*) \cdot \left[\frac{1}{n} \sum_{i=1}^n \phi(X_i; \theta) - \frac{1}{n}\sum_{i=1}^n \intX \phi(y; \theta)k_{\theta}^m(X_i, y)dy\right] \ge \kappa \Vert \theta - \theta^*\Vert^2 \label{eqn: Yuillie2}
\end{align}
for some $\kappa > 0$. However, they may not be appropriate for rigorous convergence results as they involve data samples $X_i$ in the LHS and non-random quantities in the RHS. In particular, (\ref{eqn: Yuillie1}) holds with probability 0 if $X_i$ is a continuous random variable and $f_\theta: x \mapsto \intX k_\theta^m(x,y)\phi(y; \theta)dy$ is a continuous function. Also, the term in the brackets in the LHS of (\ref{eqn: Yuillie2}) is expected to be $\Omega(1/\sqrt{n})$ by intuitions from large sample theory, and thus (\ref{eqn: Yuillie2}) may not hold when $\Vert \theta -\theta^*\Vert = \mathcal{O}(1/\sqrt{n})$.

Of particular interest is the convergence property of CD in an exponential family, in which the energy function $E(x; \theta)$ has a particular form $E(x; \theta) = -\theta \cdot \phi(x) - \log c(x)$ with some function $\phi$ a.k.a. sufficient statistic and $c$ a.k.a. carrier measure. An example is the fully-visible Boltzmann Machine. Since $-\nabla_\theta E(x; \theta) = \phi(x)$ does not depends on $\theta$ in an exponential family, the CD gradient becomes
\begin{equation} \label{eqn: CD gradient, exponential family}
\gcd{\theta} = \frac{1}{n}\sum_{i=1}^n \phi(X_i) - \frac{1}{n}\sum_{i=1}^n \phi(X_i^{(m)}).
\end{equation}

\cite{sutskever2010convergence} has shown that for Restricted Boltzmann Machines $g_\text{cd}$ is not the gradient of any function. Also, \cite{hyvarinen2006consistency} has shown that for fully-visible Boltzmann Machines the expectation of $g_\text{cd}$ is the gradient of some pseudo-likelihood function if CD using $m=1$ and Gibbs sampling with random scan. However, these interpretations, while useful, do not lead to any convergence result of CD-$m$.

Most existing theoretical studies of CD did not clearly distinguish the behavior of the estimates $\{\theta_t\}_{t \ge 0}$ in the limits of $t \to \infty$ from that of $n\to\infty$. Since in practice the CD update equation (\ref{alg: CD gradient ascent}) is iterated many times ($t \to \infty$) to obtain an estimate based on a particular data sample of size $n$, it is essential to first analyze the behavior of CD estimates $\{\theta_t\}_{t \ge 0}$ in the limit of $t \to \infty$ with fixed $n$, and then let the sample size $n \to \infty$. To fully understand the convergence property of CD, one needs to answer the following fundamental questions.
\begin{itemize}
\item Conditional on a data sample of size $n$, whether or under what conditions does CD converge to some limit point (which may depend on the data sample) as $t \to \infty$?
\item Do the above-mentioned CD limit points (in the limit of $t \to \infty$) converge to the true parameter as the sample size $n \to \infty$? If yes, what is the convergence rate? How does $m$ affect the convergence rate?
\end{itemize}

Recently we answered the above questions for CD with fixed learning rate $\eta_t = \eta$ in exponential family, by relating it to Markov chain theory and stochastic stability literature \cite{wu2016convergence}. We showed that $\lim_{t \to \infty} \sum_{s=0}^{t-1} \theta_s/t$ exists conditional on a particular data sample $\mathbf{X}_1^n$ and that CD converges to the true parameter $\theta^*$ in the sense that  $\lim_{n \to \infty} \mathbb{P}\left(\delta_n(\mathbf{X}_1^n) \ge K_m n^{-(1-2\gamma)/3}\right) = 0$, where $\delta_n(\mathbf{X}_1^n) = \Vert \lim_{t \to \infty} \sum_{s=0}^{t-1} \theta_s/t - \theta^*\Vert$, $\gamma$ is any number between 0 and $1/2$, and the coefficient factor $K_m$ depends on $m$.

Here we extend our previous work and study the convergence property of CD under annealed learning rate $\eta_t$. This is an important issue since CD in practice is dealt with anneal learning rate e.g. $\eta_t = \eta_0/t$. The argument in our previous paper relies on the critical fact that $\{\theta_t\}_{t \ge 0}$ is a homogeneous Markov chain conditional on a particular data sample, if the learning rate is fixed. In this paper, the annealing schedule of $\eta_t$ and its consequence of the unavailable homogeneous Markov property make the mathematical proof much harder than that for CD under fixed learning rate. We have to apply results from super-martingale theory.


Sections 2 states the assumptions on exponential family, MCMC kernels and learning rate, and our main result analogous to that in \cite{wu2016convergence}: all limit points (in the limit of $t \to \infty$) of
\begin{equation} \label{eqn: CD estimate}
\bar{\theta}_t = \sum_{s=0}^t \eta_s\theta_s / \sum_{s=0}^t \eta_s
\end{equation}
converge to the true parameter $\theta^*$ at a speed arbitrarily slower than $1/\sqrt[3]{n}$ as $n \to \infty$. That is, let
\begin{equation} \label{eqn: limit point distance}
\delta_n(\mathbf{X}_1^n) = \limsup_{t \to \infty} \Vert \bar{\theta}_t - \theta^*\Vert
\end{equation}
then
\begin{equation} \label{eqn: cubic convergence rate}
\lim_{n \to \infty} \mathbb{P}\left(\delta_n(\mathbf{X}_1^n) \ge K_m n^{-(1-2\gamma)/3}\right) = 0
\end{equation}
where $\gamma$ is any number between 0 and $1/2$, and the coefficient factor $K_m$ depends on $m$. This result (\ref{eqn: cubic convergence rate}) is still true if we drop the first $t_0$ parameter estimates $\{\theta_t\}_{0 \le t < t_0}$ and redefine $\bar{\theta_t}$ and $\delta_n$ in (\ref{eqn: CD estimate}) and (\ref{eqn: limit point distance}) by letting the summations from $s=0$ to $t$ start from $s = t_0$. However, for the aesthetics of the mathematical proof, we let the summations start from $s = 0$.

The remaining part of this paper is organized as follows. Sections 3-4 restate some results in \cite{wu2016convergence}, which are preliminaries for this paper. Section 3 introduces two constraints on the data sample $\mathbf{x}$\footnote{For abbreviation of notations, we write $\mathbf{X}_1^n$ as $\mathbf{X}$ and $\mathbf{x}_1^n$ as $\mathbf{x}$ in the remaining part of the paper.}, and Section 4 bounds the bias of CD gradient under  $\mathbb{P}^\mathbf{x}$, the conditional probability measure given a particular realization of data sample $\mathbf{x}$. Sections 5 and 6 construct two super-martingales under $\mathbb{P}^\mathbf{x}$, and study the limiting behavior of $\theta_t$ as $t \to \infty$. Section 7 completes the proof of the main result. Section 8 provides experimental results of CD on a fully-visible $2\times 2$ Boltzmann Machine, which demonstrate our theoretical results.

\section{Main Result}
Consider an exponential family over $\mathcal{X} \subseteq \mathbb{R}^p$ with parameter $\theta \in \Theta \subset \mathbb{R}^d$
\[
p_{\theta}(x) = c(x) e^{ \theta \cdot \phi(x) - \Lambda(\theta)}
\]
satisfying the follow assumptions.
\begin{enumerate}[label=(A\arabic*)]
\item \label{assumption: X} The sufficient statistic $\phi(x)$ is bounded, i.e. $\max_{j=1}^d \sup_{x \in \mathcal{X}} |\phi_j(x)| \leq C$ for some $C$.
\item \label{assumption: Theta} $\Theta \subseteq \mathbb{R}^d$ is compact and contains the true parameter $\theta^*$ as an interior point.
\item \label{assumption: concavity} For any $\theta \in \Theta$, $\phi_j(X), 1\le j \le d$ are linearly independent under $p_\theta$, which results in the positive definiteness of $\nabla^2 \Lambda(\theta) = \text{Cov}_\theta [\phi(X)]$.
\end{enumerate}

\ref{assumption: X} and \ref{assumption: Theta} imply $\Lambda(\theta) < \infty$ for any $\theta \in \mathbb{R}^d$ and function (\ref{eqn: L}) is well defined. This function is continuously differentiable on compact $\Theta$, and thus Lipchitz continuous. Denote $L$ denote its Lipchitz constant.
\begin{equation} \label{eqn: L}
\theta \in \Theta \mapsto \sqrt{e^{-2\Lambda({\theta^*})+\Lambda(\theta) + \Lambda(2{\theta^*} - \theta)} - 1}.
\end{equation}
\ref{assumption: Theta} and \ref{assumption: concavity} together with continuity of $\nabla^2 \Lambda(\theta)$ immediately imply that smallest eigenvalues $\lambda(\theta)$ are bounded away from 0
\begin{equation} \label{eqn: lambda}
\lambda:= \inf_{\theta \in \Theta} \lambda(\theta) > 0.
\end{equation}

We assume \ref{assumption: operator continuity} and \ref{assumption: spectral gap} for Markov transition kernels $k_\theta$ used by CD.
\begin{enumerate}[label=(A\arabic*)]
\setcounter{enumi}{3}
\item \label{assumption: operator continuity} Denoting by $\rho$ the metric on the set of Markov transition kernels $\{k_\theta: \theta \in \Theta\}$
\[
\rho(k_\theta, k_{\theta'}) := \sup_{f: f \neq 0, |f| \leq 1} \sup_{ x \in \mathcal{X}} \left|\intX f(y)k_\theta(x,y)dy - \intX f(y)k_{\theta'}(x,y)dy\right|,
\]
assume the existence of $\zeta$ such that $\rho(k_\theta, k_{\theta'}) \le \zeta \Vert \theta - \theta'\Vert$.
\item \label{assumption: spectral gap} Markov operators associated with $k_\theta$ have $\mathcal{L}_2$-spectral gap\footnote{See definition and more details in \cite{rudolf2011explicit}} $1-\alpha(\theta)>0$ and $\alpha := \sup_{\theta \in \Theta} \alpha(\theta) < 1$.
\end{enumerate}

The intuition behind \ref{assumption: operator continuity} is that, for similar $\theta$, MCMC uses similar transition kernels $k_\theta$ which lead to similar one-step transitions of every bounded function $f$. As far as we know, it is commonly obeyed by MCMC transition kernels used by CD in practice. A more general condition is that the covering number\footnote{See definition and more details in \cite{van1996weak}} $N(\epsilon, \{k_\theta: \theta \in \Theta\}, \rho) = \mathcal{O}(\epsilon^{-l})$ for some $l > 0$. \ref{assumption: spectral gap} requires all MCMC kernels mix the chains quickly. Note that MCMC algorithms such as Metropolis-Hasting and Gibbs sampling with random scan generate uniform ergodic, reversible Markov chains under mild conditions \cite{roberts2004general}, and such Markov chains have $\mathcal{L}_2$-spectral gaps \cite{paulin2012concentration}. An example satisfying assumptions \ref{assumption: X}, \ref{assumption: Theta}, \ref{assumption: concavity}, \ref{assumption: operator continuity} and \ref{assumption: spectral gap} is Gibbs sampling with random scan for fully-visible Boltzmann Machine. Details are provided in Section 8.

The last condition is imposed on the annealed learning rate $\eta_t$. It is slightly stronger than being ``not summable but square summable", as it not only requires $\sum_{s=0}^t \eta_s \to \infty$ but also requires $\sum_{s=0}^t \eta_s$ growing faster than $\sqrt{\log t}$. The popular choice $\eta_t = \eta_0/t$ satisfies this condition, for example.
\begin{enumerate}[label=(A\arabic*)]
\setcounter{enumi}{5}
\item \label{assumption: learning rate}
$\lim_{t \to \infty} \sum_{s=0}^t \eta_s / \sqrt{\log t} = \infty$ and $\sum_{t=0}^\infty \eta_t^2 < \infty$
\end{enumerate}

We confess that a problem arises from the limitation of na\"ive gradient descent method and boundedness of $\Theta$ assumed in \ref{assumption: Theta}: if $\theta_t$ is close to the boundary of $\Theta$, the CD update $\theta_{t+1} = \theta_t + \eta_t \gcd{\theta_t}$ may go outside $\Theta$. To avoid this, we let $\theta_{t+1} = \theta_t$ for those $\theta_t$ near the boundary. Formally speaking, denoting by $\partial \Theta$ the boundary of the parameter space $\Theta$, and letting
\[
\partial \Theta_t = \{\theta \in \Theta: \inf_{\theta' \in \partial \Theta} \Vert \theta - \theta' \Vert \leq 2\eta_t\sqrt{d}C\},
\]
we modify the CD update equation as
\begin{equation} \label{alg: CD gradient ascent}
\theta_{t+1} = \theta_t + \eta_t \gcd{\theta_t}\mathbb{I}\left(\theta_t \not \in \partial \Theta_t\right).
\end{equation}
Noting that $\Vert \gcd{\theta} \Vert \leq \Vert \sum_{i=1}^n \phi(X_i)/n \Vert + \Vert \sum_{i=1}^n \phi(X_i^{(m)})/n\Vert \leq 2\sqrt{d}C$, it is impossible for $\theta_t \not \in \partial \Theta_t$ to move more than $2\eta_t\sqrt{d}C$ distance towards the boundary $\partial \Theta$. Also, $\theta_t$ cannot stay at some interior point of $\Theta$ forever since $\partial \Theta_t$ gradually shrinks to $\partial \Theta$.

Now we give our main result in Theorem \ref{theorem: main}.
\begin{theorem}\label{theorem: main}
Assume \ref{assumption: X}, \ref{assumption: Theta}, \ref{assumption: concavity}, \ref{assumption: operator continuity}, \ref{assumption: spectral gap}, \ref{assumption: learning rate}. Suppose CD-$m$ algorithm in (\ref{alg: CD gradient ascent}) generates a sequence $\{\theta_t\}_{t \ge 0}$ from an i.i.d. data sample $X_1,\dots,X_n \sim p_{\theta^*}$, and $\bar{\theta}_t$ and $\delta_n(\mathbf{X})$ are defined as (\ref{eqn: CD estimate}) and (\ref{eqn: limit point distance}), respectively. If $m$ is large enough such that $\lambda - \sqrt{d}CL\alpha^m>0$ and $\theta_t \in \partial \Theta_t$ happens finitely many times, then all limit points of $\bar{\theta}_t$ converge to $\theta^*$ at a rate of $n^{-(1-2\gamma)/3}$ for any $\gamma \in (0,1/2)$, and $m$ only affects the coefficient factor $K_m$. That is,
\[
\lim_{n \to \infty} \mathbb{P}\left(\left.\delta_n(\mathbf{X}) \ge K_m n^{-(1-2\gamma)/3} \right| \theta_t \in \partial \Theta_t \text{ finitely often}\right) = 0.
\]
\end{theorem}
This theorem asserts the existence of finite $m$ such that $\bar{\theta_t}$, the weighted average of $\{\theta_t\}_{t \ge 0}$, is an consistent estimate as long as the gradient ascent does not get $\theta_t$ stuck on the boundary of the parameter space.

\section{Conditioning on Data Sample}
Having a close look at the CD algorithm, we find that the sequence $\{\theta_t\}_{t \ge 0}$ is an inhomogeneous Markov chain conditional on a realization of the data sample $\mathbf{X}= \mathbf{x}$. This result is formally stated in Lemma \ref{lemma: Markov chain}, whose proof is provided in Appendix.
\begin{lemma} \label{lemma: Markov chain}
Denote by $\mathbb{P}^\mathbf{x}$ the conditional probability measure given a certain realization of data sample $\mathbf{x}$. For any (Borel) $A \subseteq \Theta$, $\Px{\theta_{t+1} \in A |\theta_t,\dots,\theta_0} = \Px{\theta_{t+1} \in A |\theta_t}$.
\end{lemma}

We next impose two constraints (\ref{eqn: MLE}) and (\ref{eqn: empirical process}) on the data sample $\mathbf{X} \overset{\text{i.i.d.}}{\sim} p_{{\theta^*}}$ and show in Lemma \ref{lemma: data sample constraints} that both of them hold asymptotically in the limit of $n \to \infty$ with probability $1$. In the following sections, these two constraints on $\mathbf{X}=\mathbf{x}$ allow us to bound the bias of CD gradient under $\mathbb{P}^\mathbf{x}$, and results in the construction of two super-martingales.
\begin{lemma} \label{lemma: data sample constraints}
(Lemma 4.1 in \cite{wu2016convergence}) Assume \ref{assumption: X}, \ref{assumption: Theta}, \ref{assumption: concavity}, \ref{assumption: operator continuity}, and $X_1,...,X_n \sim p_{\theta^*}$ i.i.d.. Denote by $\hat{\theta}_n$ the MLE. Then for any $\gamma \in (0,1/2)$,
\begin{align}
& \sqrt{n}\Vert \hat{\theta}_n(X_1,\dots,X_n) - {\theta^*}\Vert < n^{\gamma} \label{eqn: MLE}\\
& \sup_{\theta \in \Theta} \sqrt{n}\left\Vert \frac{1}{n}\sum_{i=1}^n \intX \phi(y)k_\theta^m(X_i,y)dy - \intX \phi(y)k_\theta^m p_{{\theta^*}}(y) dy \right\Vert < n^{\gamma}. \label{eqn: empirical process}
\end{align}
hold asymptotically with probability 1. That is, $\lim_{n \to \infty} \mathbb{P}\left(X_1,\dots,X_n \text{ satisfy }(\ref{eqn: MLE}), (\ref{eqn: empirical process})\right) = 1$.
\end{lemma}

It follows from standard theorems for MLE \cite{lehmann1991theory} that (\ref{eqn: MLE}) holds asymptotically with probability 1. Letting $f_\theta: x \mapsto \intX \phi(y)k_\theta^m(x,y)dy$, (\ref{eqn: empirical process}) bounds the deviation of $\sum_{i=1}^n f_\theta(X_i)/n$ from its expectation $\mathbb{E} f_\theta(X_1) = \intX \phi(y)k_\theta^m p_{\theta^*}(y) dy$. We have to bound such deviations uniformly for all $\theta \in \Theta$ such that the bound is applicable to each of $\{\theta_t\}_{t \ge 0}$. Empirical process theory \cite{van1996weak} guarantees the concentration of $\sum_{i=1}^n f_\theta(X_i)/n$, by relating it to the covering number of function class $\{f_\theta: \theta \in \Theta\}$. A detailed proof can be found in \cite{wu2016convergence}.

\section{Bias of CD Gradient under $\mathbb{P}^\mathbf{x}$} 
This section studies the chain $\{\theta_t\}_{t \ge 0}$ under the conditional probability measure $\mathbb{P}^\mathbf{x}$ with $\mathbf{x}$ satisfying (\ref{eqn: MLE}) and (\ref{eqn: empirical process}). We have Lemma \ref{lemma: gradient bias} to bound the bias of CD gradient $\gcd{\theta}$ compared to the exact gradient $g(\theta) = \sum_{i=1}^n \phi(X_i)/n - \nabla \Lambda(\theta)$.
\begin{lemma}\label{lemma: gradient bias}
(Part of Lemma 5.1 in \cite{wu2016convergence}) Assume \ref{assumption: X}, \ref{assumption: Theta} and \ref{assumption: spectral gap} and that data sample $\mathbf{x}$ satisfies (\ref{eqn: MLE}) and (\ref{eqn: empirical process}). Denote by $\mathbb{E}^{\mathbf{x}}$ the expectation with respect to $\mathbb{P}^{\mathbf{x}}$. Then the bias of $g_\text{cd}$ has bounded magnitude
\begin{equation} \label{eqn: gradient bias}
\left\Vert \Ex{\gcd{\theta}-g(\theta)|\theta} \right\Vert
\leq  \left(1 + \sqrt{d}CL\alpha^m \right) n^{-1/2+\gamma} + \sqrt{d} CL\alpha^m \Vert \theta - \hat{\theta}_n\Vert
\end{equation}
where $\hat{\theta}_n$ is the MLE, $L$ is the Lipchitz constant of function (\ref{eqn: L}), $1-\alpha$ is the $\mathcal{L}_2$-spectral gap of MCMC operators $K_\theta$ in \ref{assumption: spectral gap}, and $\gamma \in (0,1/2)$ is introduced by (\ref{eqn: MLE}) and (\ref{eqn: empirical process}).
\end{lemma}
The idea is to decompose the bias into two parts
\begin{align*}
& \Ex{\gcd{\theta}-g(\theta)|\theta}
= \nabla \Lambda (\theta) - \frac{1}{n}\sum_{i=1}^n \intX \phi(y)k^m_\theta(x_i,y)dy\\
&=\left[\intX \phi(y)k_\theta^m p_{{\theta^*}}(y) dy - \frac{1}{n}\sum_{i=1}^n \intX \phi(y)k^m_\theta(x_i,y)dy \right] +\left[\nabla \Lambda (\theta) - \intX \phi(y)k_\theta^m p_{\theta^*}(y) dy\right].
\end{align*}
The first term has been bounded by $n^{-1/2+\gamma}$ in (\ref{eqn: empirical process}). For the second term, we use the fact that $p_\theta$ is the invariant distribution of transition kernel $k_\theta$ and write $\nabla \Lambda(\theta) = \intX \phi(y)p_{\theta}(y) dy = \intX \phi(y)k_\theta^m p_{\theta}(y) dy$. So the second term amounts to the error caused by the MCMC transitions starting from $p_{\theta^*}$ rather than the invariant distribution $p_\theta$. This error exponentially decays as $m$ increases, if $k_\theta$ has $\mathcal{L}_2$-spectral gap as assumed in \ref{assumption: spectral gap}. A detailed proof can be found in \cite{wu2016convergence}.

\section{Super-martingale Construction under $\mathbb{P}^\mathbf{x}$}
Lemma \ref{lemma: quadratic drift} studies the iterated decrement of $h^2(\theta_t) = \Vert \theta_{t} - \hat{\theta}_n \Vert^2$, and asserts that the dominant term of expected decrement $a_m h^2(\theta_t) - b_{n,m}h(\theta_t)$ is an opening-up quadratic function in $h(\theta_t) = \Vert \theta_t - \hat{\theta}_n\Vert$ if $m$ is large enough such that $a_m = \lambda - \sqrt{d}CL\alpha^m > 0$. The proof is provided in Appendix.
\begin{lemma} \label{lemma: quadratic drift}
Assume \ref{assumption: X}, \ref{assumption: Theta}, \ref{assumption: concavity},  \ref{assumption: spectral gap} and that data sample $\mathbf{x}$ satisfies (\ref{eqn: MLE}) and (\ref{eqn: empirical process}). Then $h(\theta) = \Vert \theta-\hat{\theta}_n\Vert$ satisfies
\begin{equation} \label{eqn: quadratic drift}
\Ex{h^2(\theta_{t+1}) | \theta_t } \leq h^2(\theta_t) - 2\eta_t \left[a_m h^2(\theta_t) - b_{n,m}h(\theta_t) \right] \mathbb{I}\left(\theta_t \not \in \partial \Theta_t\right) + 4d\eta_t^2C^2
\end{equation}
where
\[
a_m = \lambda -\sqrt{d}CL\alpha^m, \quad b_{n,m} = (1+\sqrt{d}CL\alpha^m)n^{-1/2+\gamma},
\]
$\lambda$ is defined in (\ref{eqn: lambda}) and $L$ is the Lipchitz constant of function (\ref{eqn: L}).
\end{lemma}

Denoting by $B$ the ball centering at the MLE $\hat{\theta}_n$ of radius $\beta b_{n,m}/ a_m$ for some $\beta>1$. As shown in Figure~\ref{fig: drift}, the expected decrement is at least $+\beta(\beta-1)b_{n,m}^2/a_m$ if $\theta_t \in B$, and it may be negative but lower bounded by $-b_{n,m}^2/4a_m$ if $\theta_t \in B^c$. Splitting apart the decrements inside/outside of $B$ constructs two super-martingales in Lemma \ref{lemma: super-martingales}. The proof is provided in Appendix.

\begin{figure}[h!] 
\center
@Misc{•,
OPTkey = {•},
OPTauthor = {•},
OPTtitle = {•},
OPThowpublished = {•},
OPTmonth = {•},
OPTyear = {•},
OPTnote = {•},
OPTannote = {•}
}
\includegraphics[width=0.9\textwidth]{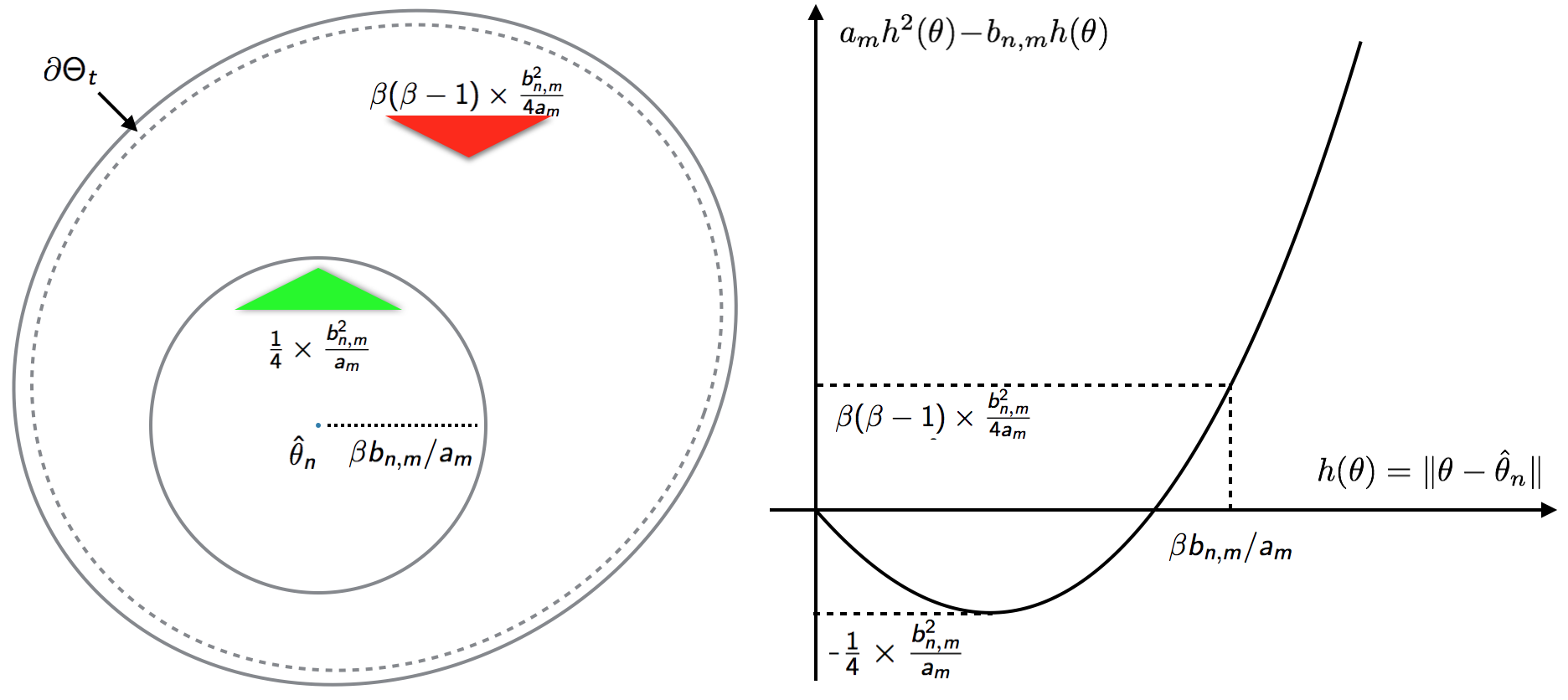}
\caption{Intuition of Lemmas \ref{lemma: quadratic drift} and \ref{lemma: super-martingales}.}\label{fig: drift}
\end{figure}

\begin{lemma} \label{lemma: super-martingales}
For any $\beta>1$, let $B = \{\theta \in \Theta: h(\theta) \le \beta b_{n,m} / a_m \}$. If $a_m = \lambda - \sqrt{d}CL\alpha^m > 0$ then it follows from Lemma \ref{lemma: quadratic drift} that both
\[
\left\{\sum_{s=0}^{t-1} Y_{s+1}\mathbb{I}\left(\theta_s \in \partial \Theta_s^c \cap B^c\right)\right\}_{t \ge 0} \quad \left\{\sum_{s=0}^{t-1} Z_{s+1}\mathbb{I}\left(\theta_s \in \partial \Theta_s \cup B\right)\right\}_{t \ge 0}
\]
are super-martingales (adapted to the natural filtration of $\{\theta_t\}_{t \ge 0}$) under $\mathbb{P}^\mathbf{x}$, where $Y_t$ and $Z_t$ are defined as below.
\begin{align}
Y_{t+1} &= h^2(\theta_{t+1}) - h^2(\theta_t) + 2\eta_t\beta(\beta-1)b^2_{n,m}/a_m - 4d\eta_t^2C^2,\\
Z_{t+1} &= h^2(\theta_{t+1}) - h^2(\theta_t) - \eta_tb^2_{n,m}/2a_m - 4d\eta_t^2C^2.
\end{align}
\end{lemma}

\section{Limiting Behaviors of $\theta_t$ under $\mathbb{P}^\mathbf{x}$}
Lemma \ref{lemma: super-martingale concentration} shows that the two super-martingales constructed in Lemma \ref{lemma: super-martingales} have bounded difference $\mathcal{O}(\eta_t)$ and further by the Azuma-Hoeffding inequality \cite{azuma1967weighted} and Borel-Cantelli Lemma \cite{durrett2010probability} that
\begin{align}
\limsup_{t \to \infty} \frac{\sum_{s=0}^{t-1} Y_{s+1}\mathbb{I}\left(\theta_s \in \partial \Theta_s^c \cap B^c\right)}{\sum_{s=0}^{t-1} \eta_s} \le 0, \label{eqn: Y limsup}\\
\limsup_{t \to \infty} \frac{\sum_{s=0}^{t-1} Z_{s+1}\mathbb{I}\left(\theta_s \in \partial\Theta_s \cup B\right)}{\sum_{s=0}^{t-1} \eta_s} \le 0. \label{eqn: Z limsup}
\end{align}
$\mathbb{P}^\mathbf{x}$-almost surely. Further, (\ref{eqn: Y limsup}) and (\ref{eqn: Z limsup}) imply
\begin{equation} \label{eqn: limit distribution}
\liminf_{t \to \infty} \frac{\sum_{s=0}^t \eta_s \mathbb{I}\left(\theta_s \in \partial \Theta_s \cup B\right)}{\sum_{s=0}^t \eta_s} \ge \frac{4\beta(\beta-1)}{4\beta(\beta-1)+1},
\end{equation}
$\mathbb{P}^\mathbf{x}$-almost surely. It suggests that: if the na\"ive gradient descent update does not get $\theta_t$ stuck at the boundary of the parameter space $\Theta$, that is $\sum_{t=0}^\infty \mathbb{I}\left(\theta_t \in \partial \Theta_t\right) < \infty$, $\theta_t$ stays a large proportion of time (weighted by $\eta_t$) in the ball $B$.


\begin{lemma} \label{lemma: super-martingale concentration}
Assume \ref{assumption: X}, \ref{assumption: Theta}, \ref{assumption: concavity},  \ref{assumption: spectral gap}, \ref{assumption: learning rate} and that data sample $\mathbf{x}$ satisfies (\ref{eqn: MLE}) and (\ref{eqn: empirical process}). If $\lambda - \sqrt{d}CL\alpha^m > 0$ then (\ref{eqn: Y limsup}), (\ref{eqn: Z limsup}) and (\ref{eqn: limit distribution}) hold $\mathbb{P}^\mathbf{x}$-almost surely.
\end{lemma}
\begin{proof}
We first show the two super-martingales has bounded differences $\mathcal{O}(\eta_t)$. Write
\[
h^2(\theta_{t+1}) - h^2(\theta_t) = 2 \eta_t \gcd{\theta_t} \cdot (\theta_t - \hat{\theta}_n) + \eta_t^2 \Vert \gcd{\theta_t} \Vert^2
\]
implying $|h^2(\theta_{t+1}) - h^2(\theta_t)| = \mathcal{O}(\eta_t)$, and further $|Y_{t+1}\mathbb{I}\left(\theta_t \in \partial \Theta_t^c \cap B^c\right)| \le H\eta_t$
for some constant $H>0$. Applying Azuma-Hoeffding inequality yields for any $\epsilon>0$,
\begin{align*}
\Px{\frac{\sum_{s=0}^{t-1} Y_{s+1}\mathbb{I}\left(\theta_s \in \partial \Theta_s^c \cap B^c\right)}{\sum_{s=0}^{t-1} \eta_s} \geq \epsilon}
&\leq \exp{\left(-\frac{\epsilon^2 (\sum_{s=0}^{t-1} \eta_s)^2}{2H^2\sum_{s=0}^{t-1}\eta_s^2}\right)}
\end{align*}
As assumed in \ref{assumption: learning rate} $\sum_{s=0}^\infty \eta_s^2 < \infty$, and $\sum_{s=0}^{t-1} \eta_s / \sqrt{\log t} \to \infty$. So the RHS $\leq \exp{(-2\log t)} = 1/t^2$ for sufficiently large $t$ and thus is summable. Applying Borel-Cantelli Lemma yields
\[
\Px{\frac{\sum_{s=0}^{t-1} Y_{s+1}\mathbb{I}\left(\theta_s \in \partial \Theta_s^c \cap B^c\right)}{\sum_{s=0}^{t-1} \eta_s} \geq \epsilon \text{ infinitely often}} = 0.\]
That is,
\[
\limsup_{t \to \infty} \frac{\sum_{s=0}^{t-1} Y_{s+1}\mathbb{I}\left(\theta_s \in \partial \Theta_s^c \cap B^c\right)}{\sum_{s=0}^{t-1} \eta_s} < \epsilon
\]
$\mathbb{P}^\mathbf{x}$-almost surely. Noting that $\epsilon$ is arbitrary, we have (\ref{eqn: Y limsup}). An analogous argument obtains (\ref{eqn: Z limsup}).
Next, noting boundedness of $h(\theta_t)$ and the fact that $\eta_t$ is not summable but square summable, we have
\begin{equation} \label{eqn: Y limit}
\frac{\sum_{s=0}^{t-1} Y_{s+1}}{\sum_{s=0}^{t-1} \eta_s} = \frac{h^2(\theta_t)-h^2(\theta_0)}{\sum_{s=0}^{t-1} \eta_s} + \frac{2\beta(\beta-1)b^2_{n,m}}{a_m} - \frac{4dC^2\sum_{s=0}^{t-1} \eta_s^2}{\sum_{s=0}^{t-1} \eta_s} \to \frac{2\beta(\beta-1)b^2_{n,m}}{a_m}
\end{equation}
$\mathbb{P}^\mathbf{x}$-almost surely as $t \to \infty$. (\ref{eqn: Y limsup}), (\ref{eqn: Z limsup}), (\ref{eqn: Y limit}) and the fact that
\begin{align*}
(Y_{s+1} - Z_{s+1})\mathbb{I}\left(\theta_s \in \partial \Theta_s \cup B\right) = Y_{s+1} - Y_{s+1}\mathbb{I}\left(\theta_s \in \partial \Theta_s^c \cap B^c\right) - Z_{s+1}\mathbb{I}\left(\theta_s \in \partial \Theta_s \cup B\right)
\end{align*}
imply
\[
\liminf_{t \to \infty} \frac{\sum_{s=0}^{t-1} (Y_{s+1} - Z_{s+1})\mathbb{I}\left(\theta_s \in \partial \Theta_s \cup B\right)}{\sum_{s=0}^{t-1} \eta_s} \ge \frac{2\beta(\beta-1)b^2_{n,m}}{a_m}.
\]
We divide both sides by $(2\beta(\beta-1)+1/2)b^2_{n,m}/a_m$ and yield (\ref{eqn: limit distribution}).
\end{proof}

\section{Convergence of CD to True Parameter}
So far we have (\ref{eqn: limit distribution}) to describe the behaviors of $\theta_t$ in the limit of $t \to \infty$ conditional on a particular data sample of size $n$. Lemma \ref{lemma: limsup bound} follows to give an upper bound for $h(\bar{\theta}_t) = \Vert \bar{\theta}_t - \hat{\theta}_n\Vert$ under $\mathbb{P}^\mathbf{x}$. Such a bound decays at a rate roughly $1/\sqrt[3]{n}$ as $n \to \infty$. The key is to let $\beta$ increase with $n$ at an appropriate rate such that the radius of $B$ vanishes as $n \to \infty$ while the proportion of time of $\theta_t \in B$ increases to 1. The convergence in (unconditional) probability result in Theorem \ref{theorem: main} is a consequence of Lemma \ref{lemma: limsup bound}.

\begin{lemma} \label{lemma: limsup bound}
From Lemma \ref{lemma: super-martingale concentration}, it follows that if $\lambda - \sqrt{d}CL\alpha^m > 0$ and $\sum_{t=0}^\infty \mathbb{I}(\theta_t \in \partial \Theta_t) < \infty$,
\[
\limsup_{t \to \infty} \Vert \bar{\theta}_t - \hat{\theta}_n \Vert = \mathcal{O}(n^{-(1-2\gamma)/3})
\]
$\mathbb{P}^\mathbf{x}$-almost surely. And the coefficient factor depends on $m$ but not data sample $\mathbf{x}$.
\end{lemma}
\begin{proof}
Using the convexity of $h(\theta)$, (\ref{eqn: limit distribution} and the assumption $\sum_{t=0}^\infty \mathbb{I}(\theta_t \in \partial \Theta_t) < \infty$ yields
\begin{align*}
\limsup_{t\to\infty}h(\bar{\theta}_t) &\leq \limsup_{t\to\infty}\frac{\sum_{s=0}^t \eta_s h(\theta_s)}{\sum_{s=0}^t \eta_s} \tag*{[convexity of $h(\theta)$ in $\theta$]}\\
&= \limsup_{t\to\infty}\frac{\sum_{s=0}^t \eta_s h(\theta_s) \mathbb{I}(\theta_s \in B)}{\sum_{s=0}^t \eta_s} + \limsup_{t\to\infty}\frac{\sum_{s=0}^t \eta_s h(\theta_s) \mathbb{I}(\theta_s \not \in B)}{\sum_{s=0}^t \eta_s}\\
&\le \frac{\beta b_{n,m}}{a_m}  + \max_{\theta \in \Theta} h(\theta) \times \limsup_{t\to\infty} \frac{\sum_{s=0}^t \eta_s \mathbb{I}(\theta_s \not \in B)}{\sum_{s=0}^t \eta_s}\\
&\le \frac{\beta b_{n,m}}{a_m}  + \max_{\theta \in \Theta} h(\theta) \times \frac{1}{4\beta(\beta-1)+1}.
\end{align*}
The desired bound $\mathcal{O}(n^{-(1-2\gamma)/3})$ is obtained by letting $\beta = n^{(1-2\gamma)/6}$. If so, the radius of $B$ is $\frac{\beta b_{n,m}}{a_m} \asymp n^{(1-2\gamma)/6} \times n^{-1/2+\gamma}= n^{-(1-2\gamma)/3}$, and the proportion of time of $\theta_t \not \in B$ is at most $\frac{1}{4\beta(\beta-1)+1} \asymp n^{-(1-2\gamma)/3}$. The coefficient factor $\frac{1+\sqrt{d}CL\alpha^m}{\lambda - \sqrt{d}CL\alpha^m} + \frac{1}{4}\max_{\theta, \theta' \in \Theta} \Vert \theta - \theta'\Vert$ depends on $m$ but not the data sample $\mathbf{x}$.
\end{proof}

\begin{proof}[Proof of Theorem \ref{theorem: main}]
By Lemma \ref{lemma: data sample constraints},
\[
\lim_{n \to \infty} \mathbb{P}\left(\left. X_1,\dots,X_n \text{ satisfies (\ref{eqn: MLE}), (\ref{eqn: empirical process})} \right| \theta_t \in \partial \Theta_t \text{ finitely often} \right) = 1
\]
if $\mathbb{P}\left(\theta_t \in \partial \Theta_t \text{ finitely often} \right) > 0$. It suffices to show that for any $\mathbf{x}$ satisfying (\ref{eqn: MLE}) and (\ref{eqn: empirical process})
\[
\lim_{n \to \infty} \Px{\left. \limsup_{t \to \infty} \Vert \bar{\theta}_t - \theta^*\Vert \le K_m n^{-(1-2\gamma)/3} \right| \theta_t \in \partial \Theta_t \text{ finitely often}} = 1
\]

This is an immediate consequence of the bound $\mathcal{O}(n^{-(1-2\gamma)/3})$ given by Lemma \ref{lemma: limsup bound} and the fact that
\[
\Vert \bar{\theta}_t - \theta^*\Vert \le \Vert \hat{\theta}_n - \theta^*\Vert + \Vert \bar{\theta}_t - \hat{\theta}_n\Vert \le n^{-1/2+\gamma} + \mathcal{O}(n^{-(1-2\gamma)/3}).
\]
\end{proof}

\section{Example: CD for fully-visible Boltzmann Machine}
 An example satisfying assumptions \ref{assumption: X}, \ref{assumption: Theta}, \ref{assumption: concavity}, \ref{assumption: operator continuity} and \ref{assumption: spectral gap} is Gibbs sampling with random scan for fully-visible Boltzmann Machine (details are discussed in Appendix). To demonstrate our theoretical results, we give experimental results of CD in a fully-visible $2 \times 2$ Boltzmann Machine
\[
p_\theta(x_1,x_2) \propto \exp{\left(\left[\begin{array}{c} x_1\\x_2\end{array}\right]^T\left[\begin{array}{c c}\theta^{(1)} & \theta^{(2)}/2\\ \theta^{(2)}/2 & \theta^{(3)}\end{array}\right]\left[\begin{array}{c} x_1\\x_2\end{array}\right]\right)} = \exp{\left(\theta^{(1)}x_1^2+\theta^{(2)}x_1x_2+\theta^{(3)}x_2^2\right)}
\]
with $x_1,x_2 \in \{0,1\}$. Three data sets of size $n=10^2,10^3,10^4$ are sampled with true parameter $\theta^* = (0.5,1.0,0.5)^T$. Then CD-2 and CD-4 run $t=1000$ iterations of updates with $\eta_t = 1/t$ and generate a sequence of parameter estimates $\{\theta_t\}_{0 \le t \le 1000}$. We drop the first $t_0 = 50$ estimates\footnote{As mentioned in the introduction section, we assume $t_0 = 0$ in the theoretical part of this paper only for aesthetics of mathematical proof. The convergence results hold for any $t_0$.}, and plot in Figure \ref{fig: experiments} the distance from $\bar{\theta}_t = \sum_{s=50}^t \eta_s\theta_s/\sum_{s=50}^t \eta_s$ to the true parameter $\theta^*$.
 
\begin{figure}[h!] 
\center
\includegraphics[width=0.88\textwidth]{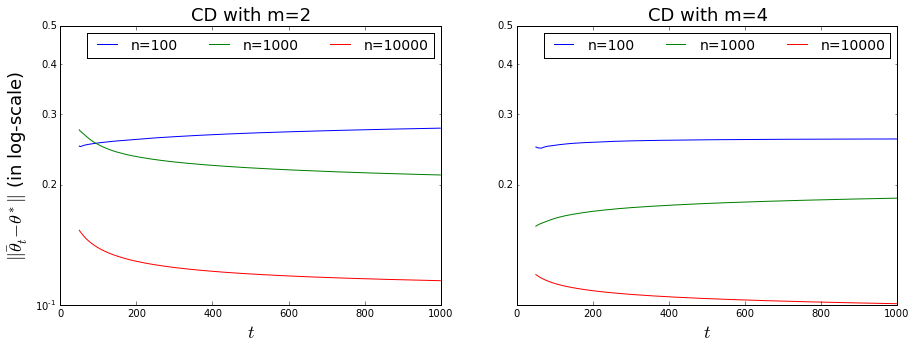}
\caption{For each fixed $n=10^2, 10^3, 10^4$, $\bar{\theta}_t$ converges to some limit point(s) as $t \to \infty$, which is not $\theta^*$. But $\delta_n = \limsup_{t \to \infty} \Vert \bar{\theta}_t - \theta^*\Vert$ decreases as $n$ increases. The effect of $m=2$ or $4$ is not significant as it only changes the coefficient factor of the convergence rate. }\label{fig: experiments}
\end{figure}

\small
\bibliographystyle{unsrtnat}
\bibliography{ref} 

\section{Appendix}
\subsection*{Proof of Lemma 3.1}
\begin{proof}
It is clearly true if $\theta_t \in \partial \Theta_t$. If $\theta_t \not \in \partial \Theta_t$, $\theta_{t+1}$ in the CD update equation (\ref{alg: CD gradient ascent}) is function of $\theta_t$, $\mathbf{X}$ and a random sample $X^{(m)}_i$. And $X^{(m)}_i \sim k_{\theta_t}^m(x_i,\cdot)$ are conditionally independent to the history of $\theta_0,\dots,\theta_{t-1}$ given $\mathbf{X}=\mathbf{x}$ and $\theta_t$. For any $A \subseteq \Theta$,
\begin{align*} \Px{\theta_{t+1} \in A|\theta_t, \dots,\theta_0}&= \P{\theta_{t+1} \in A|\theta_t, \dots,\theta_0, \mathbf{X}=\mathbf{x}}\\ &= \P{\mathbf{X}^{(m)} \in g_\text{cd}^{-1}((A - \theta_t)/\eta_t) |\theta_t, \dots,\theta_0, \mathbf{X}=\mathbf{x}}\\&= \P{\mathbf{X}^{(m)} \in g_\text{cd}^{-1}((A - \theta_t)/\eta_t) |\theta_t, \mathbf{X}=\mathbf{x}}\\&= \P{\theta_{t+1} \in A |\theta_t, \mathbf{X}=\mathbf{x}}\\&=\Px{\theta_{t+1} \in A|\theta_t}\end{align*}
\end{proof}

\subsection*{Proof of Lemma 5.1}
\begin{proof}
Lemma \ref{lemma: Markov chain} has shown that $\{\theta_t\}_{t \ge 0}$ is an inhomogeneous Markov chain. It suffices to show
\begin{align}
\Ex{Y_{t+1}\mathbb{I}\left(\theta_t \in \partial \Theta_t^c \cap B^c\right)|\theta_t} &\ge 0 \label{eqn: Y},\\
\Ex{Z_{t+1}\mathbb{I}\left(\theta_t \in \partial \Theta_t \cup B\right)|\theta_t} &\ge 0. \label{eqn: Z}
\end{align}
Indeed, if $\theta_t \in \partial \Theta_t^c \cap B^c$ then $h(\theta_t) \ge \beta b_{n,m}/a_m$ implies $a_m h^2(\theta_t) - b_{n,m}h(\theta_t) \ge \beta(\beta-1)b^2_{n,m}/a_m$, which together with (\ref{eqn: quadratic drift}) further implies
\[
\Ex{h^2(\theta_{t+1}) | \theta_t } \leq h^2(\theta_t) - 2\beta(\beta-1)\eta_tb^2_{n,m}/a_m  + 4d\eta_t^2C^2,
\]
completing the proof of (\ref{eqn: Y}). Analogously, if $\theta_t \in B$ then $h(\theta_t) \le \beta b_{n,m}/a_m$ implies $a_m h^2(\theta_t) - b_{n,m}h(\theta_t) \ge -\eta_tb^2_{n,m}/4a_m$, which together with (\ref{eqn: quadratic drift}) further implies
\[
\Ex{h^2(\theta_{t+1}) | \theta_t } \leq h^2(\theta_t) + \eta_tb^2_{n,m}/2a_m  + 4d\eta_t^2C^2.
\]
It with the fact that $\Ex{h^2(\theta_{t+1}) | \theta_t } = h^2(\theta_t)$ if $\theta_t \in \partial \Theta_t$, completes the proof of (\ref{eqn: Z}).
\end{proof}

\subsection*{Proof of Lemma 5.2}
\begin{proof}
If $\theta_t \in \partial \Theta_t$, (\ref{eqn: quadratic drift}) trivially hold. If $\theta_t \not \in \partial \Theta_t$,
\begin{align*}
h^2(\theta_{t+1}) &= h^2(\theta_{t+1} + \eta_t \gcd{\theta_t})\\
&= h^2(\theta_t) + 2\eta_t \gcd{\theta_t} \cdot (\theta_t - \hat{\theta}_n) + \eta_t^2 \Vert \gcd{\theta_t}\Vert^2\\
&= h^2(\theta_t) + 2\eta_t g(\theta_t) \cdot (\theta_t - \hat{\theta}_n) + 2\eta_t [\gcd{\theta_t} - g(\theta_t)] \cdot (\theta_t - \hat{\theta}_n) + \eta_t^2 \Vert \gcd{\theta_t}\Vert^2\\
&\le h^2(\theta_t) - 2\eta_t \lambda h^2(\theta_t) + 2\eta_t [\gcd{\theta_t} - g(\theta_t)] \cdot (\theta_t - \hat{\theta}_n) + 4d\eta_t^2C^2
\end{align*}
where the last step follows from the facts that $g(\theta_t) \cdot (\theta_t - \hat{\theta}_n) = -(\theta_t - \hat{\theta}_n)^T\nabla^2 \Lambda(\theta')(\theta_t - \hat{\theta}_n) \leq -\lambda h^2(\theta_t)$ with some $\theta'$ between $\theta_t$ and $\hat{\theta}_n$ and that $\Vert \gcd{\theta} \Vert \leq 2\sqrt{d}C$. Taking conditional expectation $\Ex{\cdot|\theta_t}$ on both sides yields
\begin{align*}
\Ex{h^2(\theta_{t+1})|\theta_t} &\le h^2(\theta_t)-2\eta_t \lambda h^2(\theta_t) + 2\eta_t \Ex{\gcd{\theta_t} - g(\theta_t)|\theta_t} \cdot (\theta_t - \hat{\theta}_n) + 4d\eta_t^2C^2\\
&\le h^2(\theta_t)-2\eta_t \lambda h^2(\theta_t) + 2\eta_t \Vert \Ex{\gcd{\theta_t} - g(\theta_t)|\theta_t}\Vert \times h(\theta_t) + 4d\eta_t^2C^2
\end{align*}
Using Lemma \ref{lemma: gradient bias} and rearranging terms yields (\ref{eqn: quadratic drift}) as desired. 
\end{proof}

\subsection*{Gibbs Sampling for Fully-visible Boltzmann Machine}
With $x \in \{0,1\}^p$ and a symmetric matrix $W_{p \times p}$, fully-visible Boltzmann Machine is given by
\[
p_W(x) \propto \exp{(x^TWx)} = \exp{\left(\sum_{j=1}^p W_{jj}x_j^2 + \sum_{1 \le j < k \le p} 2W_{jk} x_jx_k\right)}
\]
which apparently belongs to an exponential family satisfying \ref{assumption: X} and \ref{assumption: concavity}, with $d = p(p+1)/2$, $\phi(x) = (x_j^2, 1 \le j \le p; x_jx_k, 1 \le j < k \le p)$, $\theta = (W_{jj}, 1 \le j \le p; 2W_{jk}, 1 \le j < k \le p)$, $C = 1$. Let $\Theta$ to be the set of $\theta$ such that $W(\theta)$ has bounded Frobenius norm $\Vert W \Vert_\text{F} = \sqrt{\text{trace}(W^TW)} \le M$, then $\Theta$ is compact as required in \ref{assumption: Theta}. The probabilities $k_\theta(x,x')$ of the Gibbs sampler flipping $x_j \to x_j'$ are continuously differentiable in $\theta$ (or equivalent $W$) on compact set $\Theta$, and thus Lipchitz continuous in $\theta$. Therefore \ref{assumption: operator continuity} is satisfied as
\[
\rho(k_\theta, k_{\theta'}) \leq \sup_{x \in \{0,1\}^p} \sum_{x' \in \{0,1\}^p} |k_\theta(x,x')-k_{\theta'}(x,x')| \leq \zeta \Vert \theta - \theta'\Vert.
\]
for some $\zeta$. Moreover, a Gibbs sampler with random scan generates a uniform ergodic, reversible Markov chain which has $\mathcal{L}_2$-spectral gap as required in \ref{assumption: spectral gap}.
\end{document}